\documentclass[conference]{IEEEtran}
\IEEEoverridecommandlockouts

\usepackage{cite}
\usepackage{amsmath,amssymb,amsfonts}
\usepackage{algorithmic}
\usepackage{graphicx}
\usepackage{textcomp}
\usepackage{xcolor}

\usepackage{amsthm}
\newtheorem{definition}{Definition}
\newtheorem{theorem}{Theorem}
\DeclareMathOperator*{\argmax}{arg\,max}

\newcommand{\distas}[1]{\mathbin{\overset{#1}{\kern\z@\sim}}}%
\usepackage{bm}
\usepackage{booktabs}
\usepackage{multirow}
\usepackage{multicol}
\usepackage{xcolor}
\usepackage{subfig}

\usepackage{tikz}
\newcommand*\circled[1]{\tikz[baseline=(char.base)]{
            \node[shape=circle,fill,inner sep=1pt] (char) {\footnotesize \textcolor{white}{#1}};}}

\def\BibTeX{{\rm B\kern-.05em{\sc i\kern-.025em b}\kern-.08em
    T\kern-.1667em\lower.7ex\hbox{E}\kern-.125emX}}
\begin{document}

\title{Power Interpretable Causal ODE Networks: A Unified Model for Explainable Anomaly Detection and Root Cause Analysis in Power Systems\\
}

\author{
\IEEEauthorblockN{Yue Sun, Likai Wang, Rick S. Blum, Parv Venkitasubramaniam}
\IEEEauthorblockA{
\textit{Electrical Engineering, Lehigh University}\\
Bethlehem, PA, USA\\
\{yus516, liw324, rb0f, pav309\}@lehigh.edu
}
}

\maketitle

\begin{abstract}
Anomaly detection and root cause analysis (RCA) are critical for ensuring the safety and resilience of cyber-physical systems such as power grids. However, existing machine learning models for time series anomaly detection often operate as black boxes, offering only binary outputs without any explanation, such as identifying anomaly type and origin. 
To address this challenge, we propose Power Interpretable Causality Ordinary Differential Equation (PICODE) Networks, a unified, causality-informed architecture that jointly performs anomaly detection along with the explanation why it is detected as an anomaly, including root cause localization, anomaly type classification, and anomaly shape characterization. 
Experimental results in power systems demonstrate that PICODE achieves competitive detection performance while offering improved interpretability and reduced reliance on labeled data or external causal graphs.
We provide theoretical results demonstrating the alignment between the shape of anomaly functions and the changes in the weights of the extracted causal graphs. 
\end{abstract}


\section{Introduction}

Anomaly detection is a critical task in the operation and protection of modern power systems~\cite{wadi2021anomaly, cooper2023anomaly, asefi2023anomaly}. Anomalies arise from a wide variety of sources, including physical events (e.g., line faults~\cite{zhao2018anomaly, danilczyk2021smart}), cyber intrusions (e.g., FDI, also known as false data injection attacks~\cite{du2022review, madabhushi2023survey, mahi2022false}), and sensor-level issues (e.g., random noise~\cite{pan2019static}, drift~\cite{fenza2019drift}, or miscalibration~\cite{gummadi2024xai}).  These anomaly types differ widely in both their causes and consequences. For example, sensor noise is often transient and statistically benign, usually manageable through filtering and regular maintenance. Sensor drift or miscalibration, while not immediately harmful, introduces persistent bias into measurements and can degrade system performance over time. In contrast, FDI attacks are deliberate and adversarial: they aim to stealthily alter measurements in ways that mislead state estimators or trigger false control actions~\cite{chen2015detection}. Meanwhile, physical faults such as line outages can propagate rapidly and cause cascading failures if not quickly addressed~\cite{elmasry2022enhanced, sun2025unifying}.

Given this heterogeneity, it is insufficient to merely detect the presence of an anomaly merely. A binary anomaly flag provides little actionable insight unless it is accompanied by more information: What kind of anomaly is it? How urgent is it? What caused it? For instance, a transient fluctuation due to sensor noise may require no response, whereas a cyber-induced voltage instability demands immediate system-wide intervention. This distinction underscores the need for systems that not only detect anomaly but also perform root cause analysis (RCA) to identify origin and underlying mechanisms.

However, both anomaly detection and RCA are inherently challenging in power systems~\cite{chandola2009anomaly, mirzaei2009review}. Cyber attacks are often designed to mimic normal variations or hide within sensor noise~\cite{madabhushi2023survey}. Sensor drift may evolve too slowly to trigger alarms, yet can cause long-term control degradation~\cite{fenza2019drift}. Physical faults often exhibit spatial-temporal propagation, making it difficult to trace their origin~\cite{ju2021fault}. Moreover, the tight coupling of cyber and physical layers in modern grids introduces complex dynamics that invalidate many assumptions of traditional model-based approaches. These challenges underscore the need for learning-based systems that are robust, interpretable, and capable of performing joint anomaly detection and RCA.

\begin{figure}
\centering
\includegraphics[width=0.99\columnwidth]{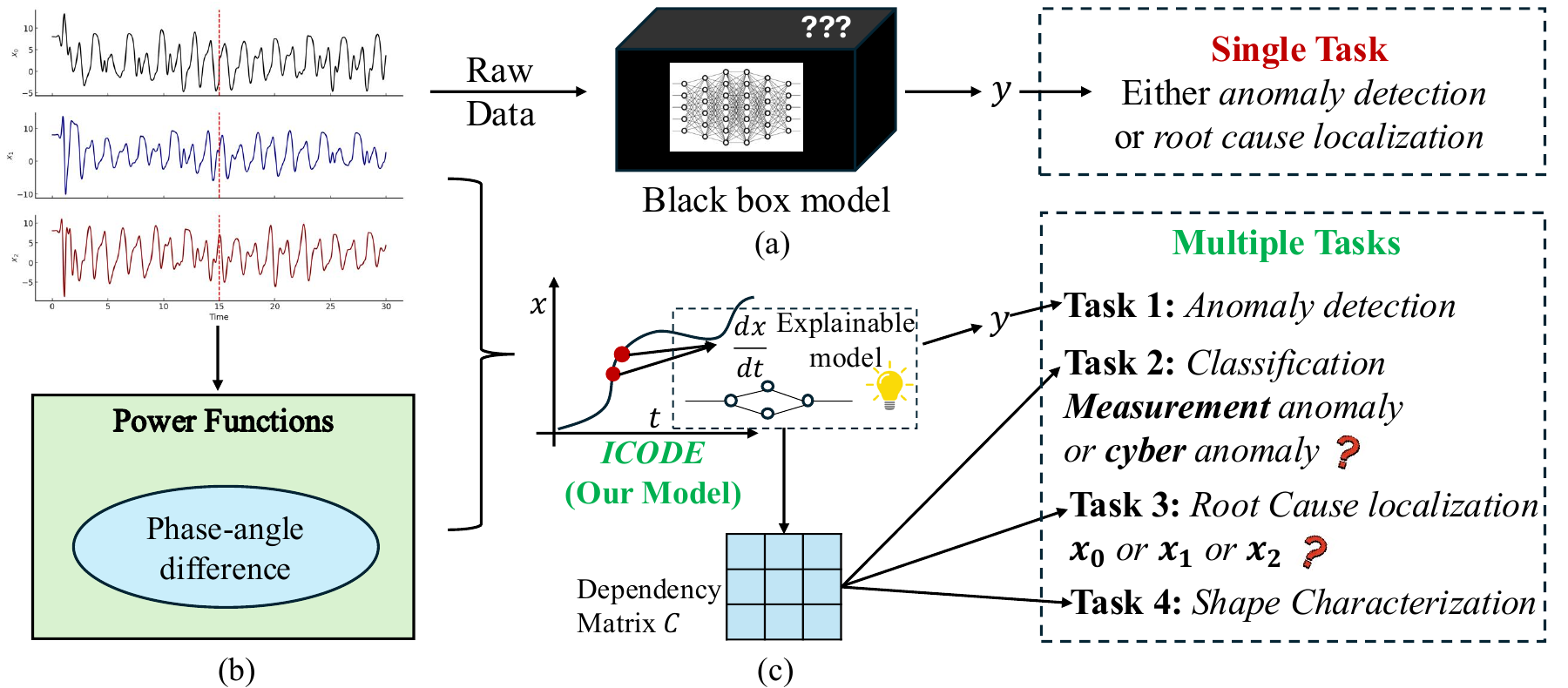} 
\caption{(a) Traditional models perform either anomaly detection or root cause localization.
(b) We encode time series using domain knowledge (e.g., phase-angle differences) for explainability.
(c) Our model unifies detection, localization, type classification, and shape characterization by leveraging a learned causal dependency graph.
}
\label{fig:introduction}
\vspace{-0.2in}
\end{figure}

Recent advances in machine learning (ML) have shown great promise in modeling complex spatiotemporal dependencies for anomaly detection~\cite{zhou2021vae, xu2021anomaly, cooper2023anomaly} and RCA~\cite{budhathoki2022causal, xin2023causalrca}. In anomaly detection, ML models typically learn normal operational patterns and flag deviations as potential anomaly~\cite{xu2021anomaly}, while RCA often involves learning causal structures to trace anomaly back to their sources~\cite{han2023root, budhathoki2022causal}. However, these two tasks are commonly treated as separate modules, resulting in increased model complexity, latency, and inconsistency between detection and attribution. Most existing approaches are also not designed to identify the type of anomaly, an essential aspect to assess urgency and implement appropriate responses. For instance, correcting a miscalibrated sensor differs significantly from responding to a coordinated cyberattack. Furthermore, these methods often require large labeled datasets and exhibit high computational cost, restricting their utility in real-time or resource-constrained settings. Besides these challenges, most ML models function as black boxes, offering little interpretability, and are purely data-driven, ignoring domain knowledge such as phase-angle dynamics~\cite{sun2025unifying}. These limitations underscore the urgent need for unified, explainable, and physics-informed AI frameworks capable of jointly performing anomaly detection, type classification, and RCA within a single model.

\begin{figure}
\centering
\includegraphics[width=0.99\columnwidth]{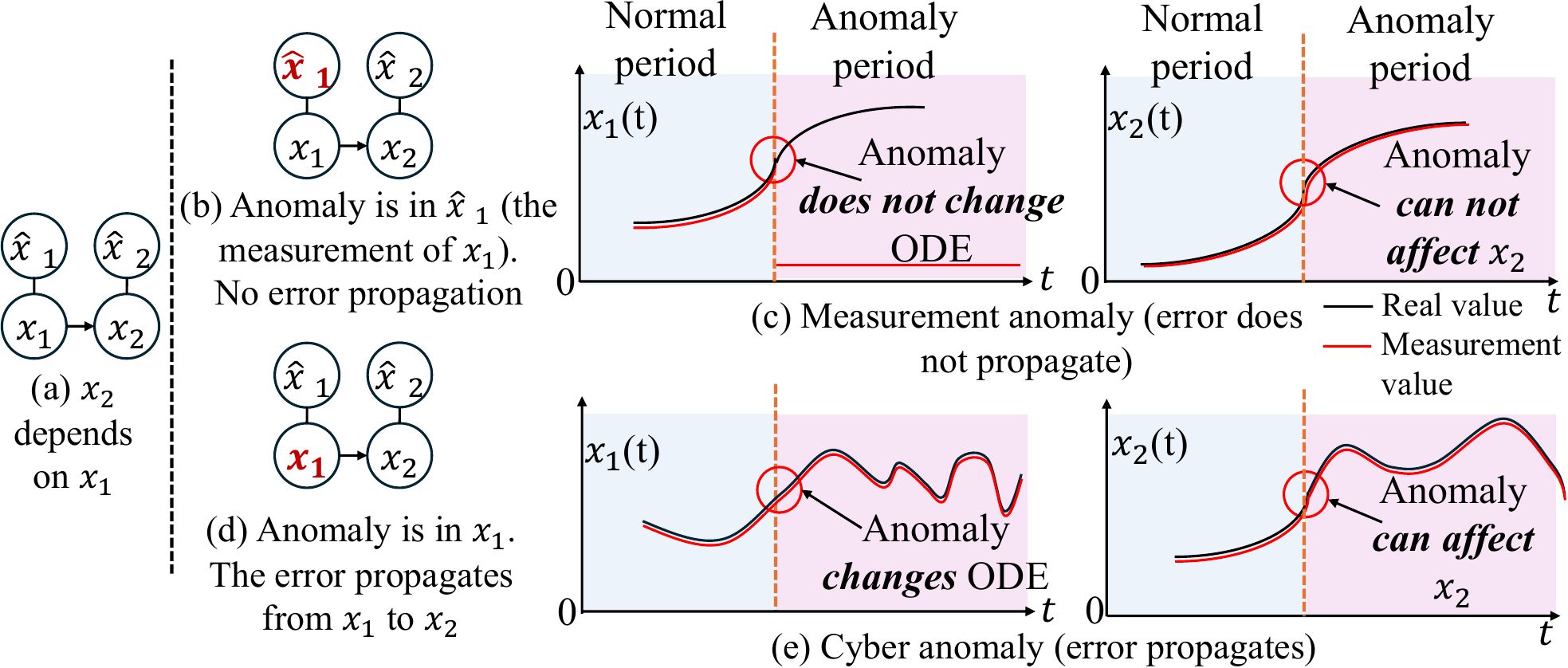} 
\caption{Comparison of Measurement and Cyber Anomalies in a Two-variable System
(a) Normal state.
(b) Measurement anomaly affects the measurement of a variable without altering underlying ODEs.
(c) Measurement anomaly does not propagate to dependent variables.
(d) Cyber anomaly alters the underlying ODE and lead to an anomaly. 
(e) Cyber anomaly propagates anomaly to dependent variables.
}
\label{fig:problem-def}
\end{figure}

To address the challenges of anomaly analysis in power systems, we begin by dividing anomaly into two types:
\begin{itemize}
    \item {Measurement anomaly}, typically caused by sensor noise, drift, or miscalibration, are generally localized to individual variables and do not alter other variables.
    \item {Cyber anomaly}, on the other hand, results from malicious interventions or hidden faults and can propagate through the system, potentially impacting multiple variables.
\end{itemize}

To investigate the differences between these types, we propose Power Interpretable Causality Ordinary Differential Equation (PICODE) which characterizes the {temporal shape} of anomaly to enhance interpretability and support operational decision-making. Specifically, we distinguish between \textit{monotone} and \textit{unimodal} patterns, which indicate whether an anomaly is consistently increasing, decreasing, or exhibiting a single peak before diminishing. This insight helps prioritize responses by differentiating between urgent, escalating events (e.g., propagating cyberattacks) and benign or self-correcting disturbances (e.g., transient noise).

Our central hypothesis is that anomaly in a dynamical system governed by ODEs manifest as perturbations to the system's governing dynamics. These perturbations induce both \textit{predictive deviations} and \textit{structural changes} in the system’s learned causal graph, as illustrated in Figure~\ref{fig:introduction}.

PICODE builds upon Neural ODEs~\cite{chen2018neural} to learn the continuous-time dynamics of power systems from observational data. By modeling normal system behavior during anomaly-free periods, PICODE can detect deviations that not only signify the presence of anomaly but also reveal root cause locations, types, and characterize the shape of the anomaly function.
The key contributions of this work are summarized below:
\begin{itemize}
    \item We define measurement and cyber anomaly based on whether the root cause remains localized or alters downstream dependencies in the causal graph.
    \item We propose {PICODE}, an interpretable machine learning framework based on Neural ODE, capable of performing anomaly detection, root cause localization, type classification, and shape characterization.
    \item We provide a theoretical analysis showing how parameter shifts in the learned ODE can be used to characterize the shape of the anomaly function.
    \item We conduct extensive experiments on three simulated ODE-based systems, demonstrating that PICODE outperforms state-of-the-art methods in anomaly detection, classification, and root cause analysis.
\end{itemize}

\section{Related Work}

\textbf{Anomaly Detection in Power Systems.}
Anomaly detection has been extensively studied in power systems due to its critical role in ensuring grid stability and operational safety~\cite{wadi2021anomaly, cooper2023anomaly, asefi2023anomaly}. Deep learning models~\cite{pang2021deep} have shown strong performance in time series anomaly detection, including power system data. Variational Autoencoder (VAE)-based models~\cite{zhou2021vae} have been used to learn compact latent representations that help distinguish normal from abnormal patterns. The Anomaly Transformer~\cite{xu2021anomaly} improves detection accuracy by jointly capturing pointwise temporal features and pairwise associations. Diffusion-based imputation models~\cite{xiao2023imputation} have also been proposed for detecting anomaly by learning data reconstruction patterns. However, most of these models are developed in general time series settings and often lack domain adaptation to power systems. Furthermore, they primarily provide binary anomaly decisions and fail to explain the nature, type, or cause of the detected anomaly. Their high model complexity and reliance on large labeled datasets also pose challenges for real-world deployment in power grids, where labeled anomaly are rare and interpretability is essential.

\begin{figure*}
\centering
\includegraphics[width=1.98\columnwidth]{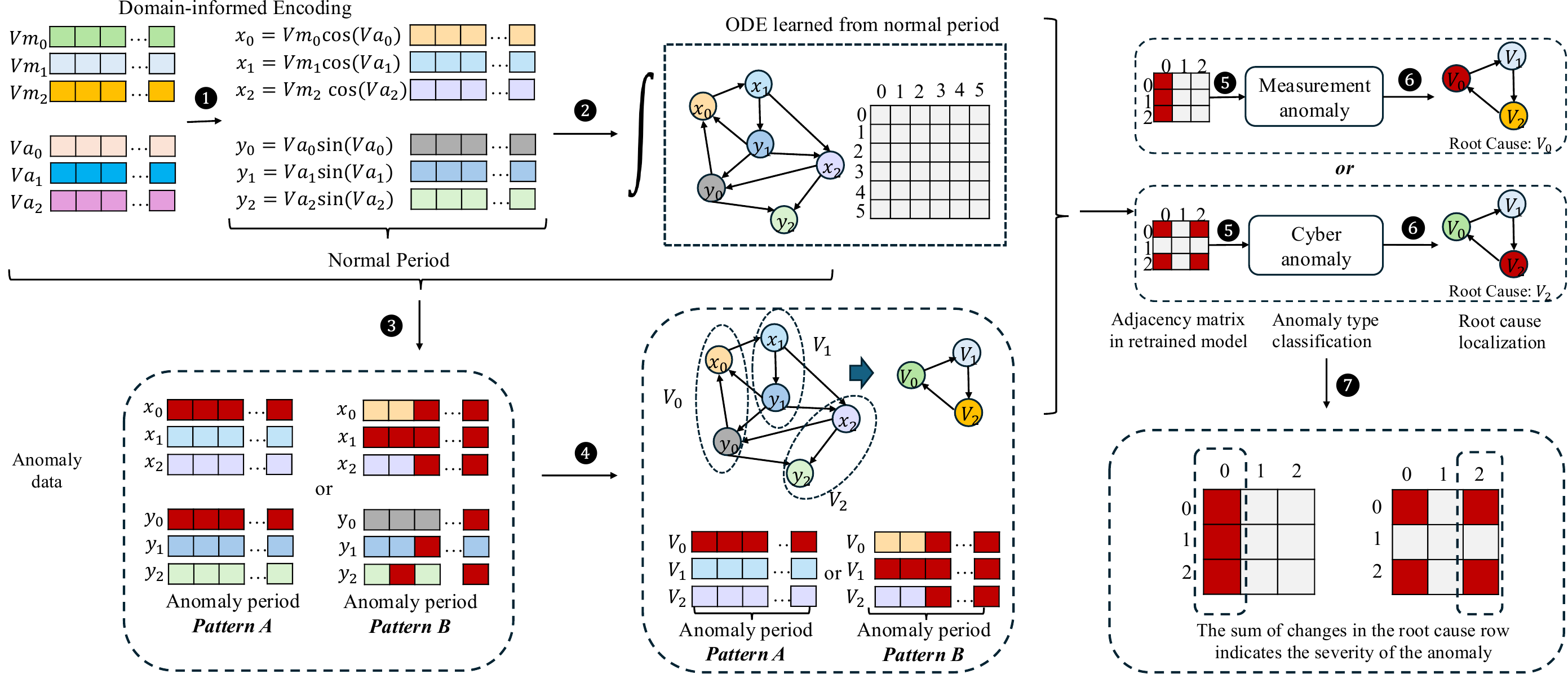} 
\caption{Illustration of PICODE framework on a 3-variable dynamical system. 
$\protect\circled{1}$ Voltage magnitude and angle measurements are encoded into domain-informed Cartesian components to form the input state $S(t)$. 
$\protect\circled{2}$ PICODE is trained on data from the normal period to learn the baseline causality matrix $C$ via an ODE-based model. 
$\protect\circled{3}$ When an anomaly is detected in new data, the model is retrained using this anomalous segment to obtain an updated causality matrix $C'$. 
$\protect\circled{4}$ Differences between $C$ and $C'$ are analyzed to identify patterns indicative of measurement or cyber anomaly. 
$\protect\circled{5}$ Based on the structural change in the causal graph, the anomaly type is classified using Eq.~(\ref{eq:anomaly-type}). 
$\protect\circled{6}$ Root cause localization is then performed using Eq.~(\ref{eq:get-root-cause}). 
$\protect\circled{7}$ Finally, the anomaly’s progression is assessed by evaluating changes in the causal influence of the root variable over time.}
\label{fig:architecture}
\end{figure*}

\textbf{Root Cause Analysis.}
RCA has emerged as a complementary task to anomaly detection, aiming to trace the origin of faults or disturbances. Machine learning-based RCA approaches~\cite{lei2020applications, richens2020improving} have gained attention across domains, including energy systems. For instance, variational inference models such as CausalRCA~\cite{han2023root} incorporate prior causal knowledge via Pearl’s Structural Causal Models to locate root causes once anomaly are detected. However, these methods typically assume access to accurate and complete causal graphs, which is difficult in complex cyber-physical systems like power grids. Moreover, most approaches treat detection and RCA as disjoint tasks, requiring two separate models, leading to higher computational cost and potential inconsistencies. Recent advances have proposed causal graph discovery methods for RCA~\cite{budhathoki2022causal, wang2023interdependent}, but many of these assume that anomaly manifest strictly as statistical outliers. This assumption does not hold for slowly drifting or adversarial anomaly, which may subtly alter system dynamics without being immediate outliers, thereby degrading the robustness of causal inference. These limitations motivate the need for unified and interpretable RCA models that do not rely on causal priors and operate directly on noisy, real-world data.

\textbf{Model Explainability using Neural ODEs.} 
Model-intrinsic-explainable models are designed with inherent transparency~\cite{sun2024generalization}, allowing direct interpretation of their decisions and processes from their internal structures~\cite{Lakkaraju2016,Lou2012, sun2024incorporating}. Self-explainable Neural Networks (SENN) employ explainable basis concepts for explicit and faithful predictions and classification~\cite{alvarez2018towards}. 
Building on SENN, the GVAR model learns causal relationships within time series data using an interpretable framework~\cite{marcinkevivcs2021interpretable}. 
Additionally, time series systems~\cite{xu2024robust, xu2025wasserstein} often follow differential equations, and Neural ODEs have been proposed to approximate these systems~\cite {chen2018neural, jia2019neural, asikis2022neural, sun2023reaction}. 
However, using model-intrinsic-explainable models enhanced by Neural ODEs to learn patterns in dynamic systems for anomaly detection and RCA remains largely unexplored.


\section{Notations and Problem definition}

The framework we propose applies to cyber-physical systems broadly. In this paper, we focus on power systems as a particular application to develop and test the computational methods. 
Power systems consist of interconnected electrical components, including generators, transmission lines, resistors, capacitors, and loads, which are used to simulate real-world operational behavior. To ensure reliable grid monitoring, sensors are deployed throughout the system to collect time series data on critical state variables such as voltage, frequency, and current. 
These variables are used for the grid control, including actions to load changes, power system stabilization, and automatic generation control, making this system a complex controller system. 
These sensor readings form the basis for real-time decision-making, anomaly detection, and system diagnostics.

In this work, we focus on two key variables commonly monitored in power systems: the \textit{voltage magnitude} \( V_m(t) = \{ V_{m1}(t), V_{m2}(t), \dots, V_{mp}(t) \} \) and the \textit{voltage angle} \( V_a(t) = \{ V_{a1}(t), V_{a2}(t), \dots, V_{ap}(t) \} \), measured at \( p \) locations over time \( t \). These variables are causally related through the physical dynamics of the power system, which can be modeled by a directed graph \( \mathcal{G} = \{ V(t), \mathcal{E} \} \), where each node represents a time-varying variable \( V_i(t) \), and each edge \( (i, j) \in \mathcal{E} \) indicates a directed influence from \( V_i \) to \( V_j \).

Many differential equations have been proven to exist connecting these variables in power systems, for example, the linearized swing equations~\cite{salam2003arnold}. These equations are usually ODEs; thus, we employ Neural ODE networks to learn the patterns. We assume the system follows ordinary differential equation (ODE) dynamics:
\begin{equation}
\label{eq:system}
\begin{aligned}
    \frac{dV_i(t)}{dt} & = f_i \left( V(t), \mathcal{G} \right), \quad \forall 1 \leq i \leq p, \\
    V_i(t) & = V_i(0) + \int_0^t f_i \left( V(s), \mathcal{G} \right) ds,
\end{aligned}
\end{equation}

\begin{figure}[!ht]
    \centering
    \includegraphics[width=0.85\columnwidth]{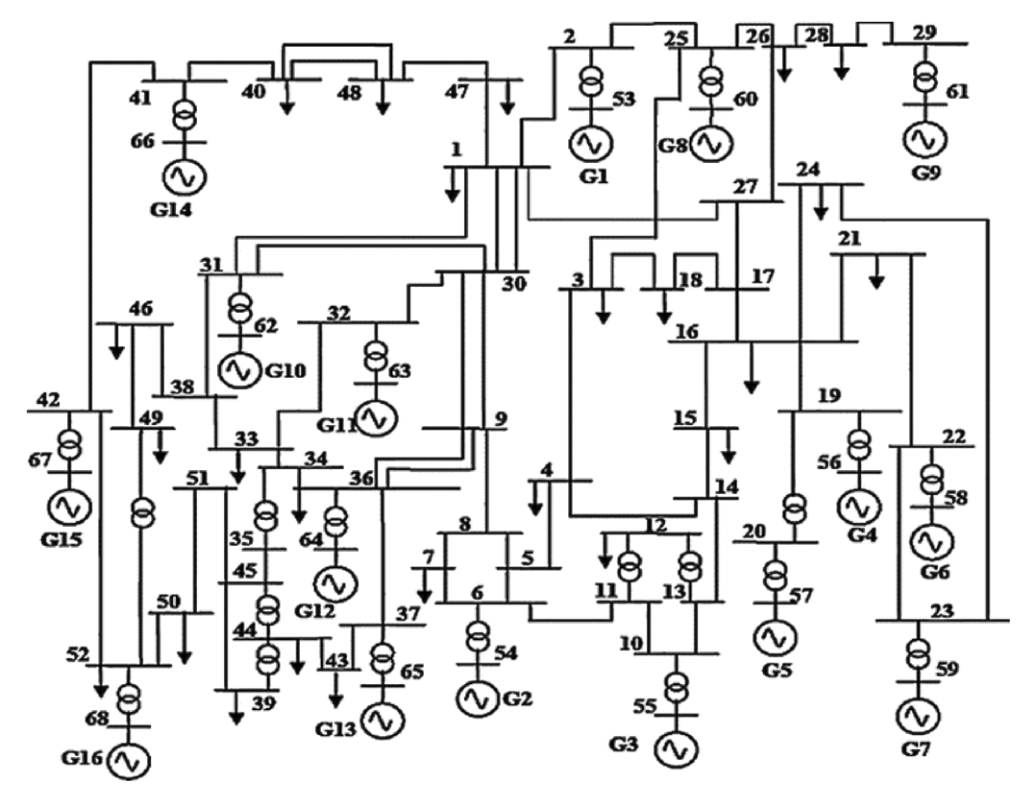}
    \caption{Single line diagram of the IEEE 68-bus test system.}
    \label{fig:ieee-diagram}
\end{figure}
\noindent{where $f_i$ is the nonlinear function that governs the evolution of variable $ V_i$ based on its causal parents in the graph.
In practice, the observed measurements \( \{ \hat{V}_1(t), \hat{V}_2(t), \dots, \hat{V}_p(t) \} \), from sensors are corrupted by local sensor faults, drift, or noise, which we define as a measurement anomaly:}

\begin{definition}
\textbf{Measurement Anomaly}: 
We define a system to exhibit a measurement anomaly in variable $i$ at time $t$, if the variable measurement at time $t$ is given by:
\begin{equation}
\label{eq:physical}
\begin{aligned}
    \hat{V}_i'(t) & = A(\hat{V}_i(t)),
\end{aligned}  
\end{equation}
where $A(\cdot)$ is an anomalous function such that $A(x) \neq x$.
\end{definition}

Unlike a measurement anomaly, a cyber anomaly modifies the true system variable \( V_i(t) \) itself, which subsequently propagates to neighboring variables. This may be due to malicious attacks, such as FDI, and is expressed as:
\begin{definition}
\textbf{Cyber Anomaly}: 
We define a system to exhibit a cyber anomaly in variable $i$ at time $t$, with anomalous function denoted as $A(\cdot)$,
if variable $v_i$ is changed to 
\begin{equation}
    V_i'(t) = A(V_i(t)), A(x) \neq x.
\end{equation} 
\end{definition}

Unlike measurement anomaly, cyber anomaly propagate through the causal structure \( \mathcal{G} \), potentially influencing downstream variables and destabilizing the system.
An example of a cyber and measurement anomaly is shown in Figure~\ref{fig:problem-def}.

In addition to identifying the type and origin of anomaly, we are also interested in their \textit{temporal dynamics}, specifically, when the system is under FDI and $A(x) = x + z (t)$, whether the magnitude of the anomaly \( |z(t)| \) is increasing, decreasing, or stable over time. We refer to this as the \textbf{shape} of anomaly.

\begin{figure}[!ht]
    \centering
    \subfloat[]{\includegraphics[width=0.49\linewidth]{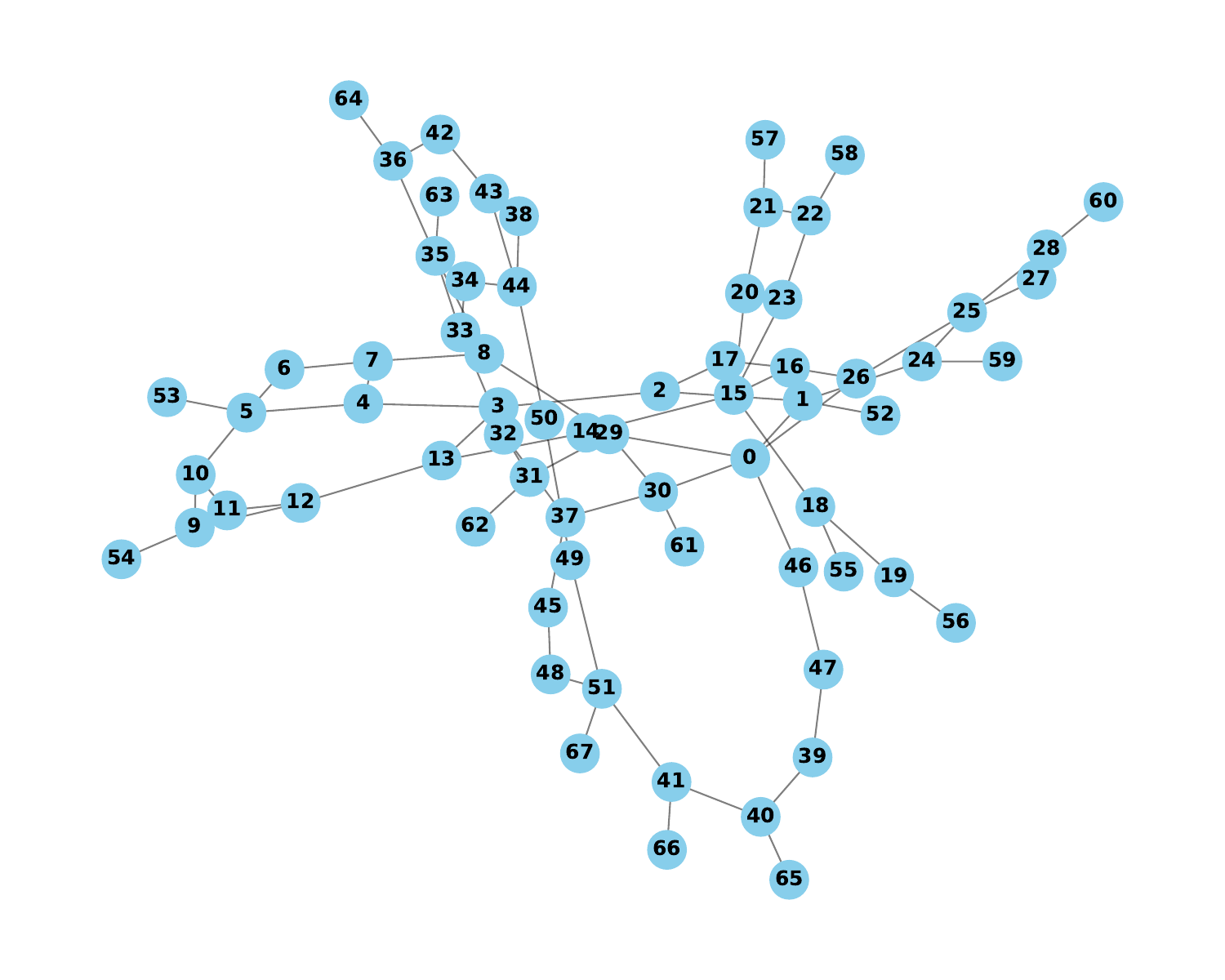}\label{fig:ieee-graph}}
    \hfill
    \subfloat[]{\includegraphics[width=0.49\linewidth]{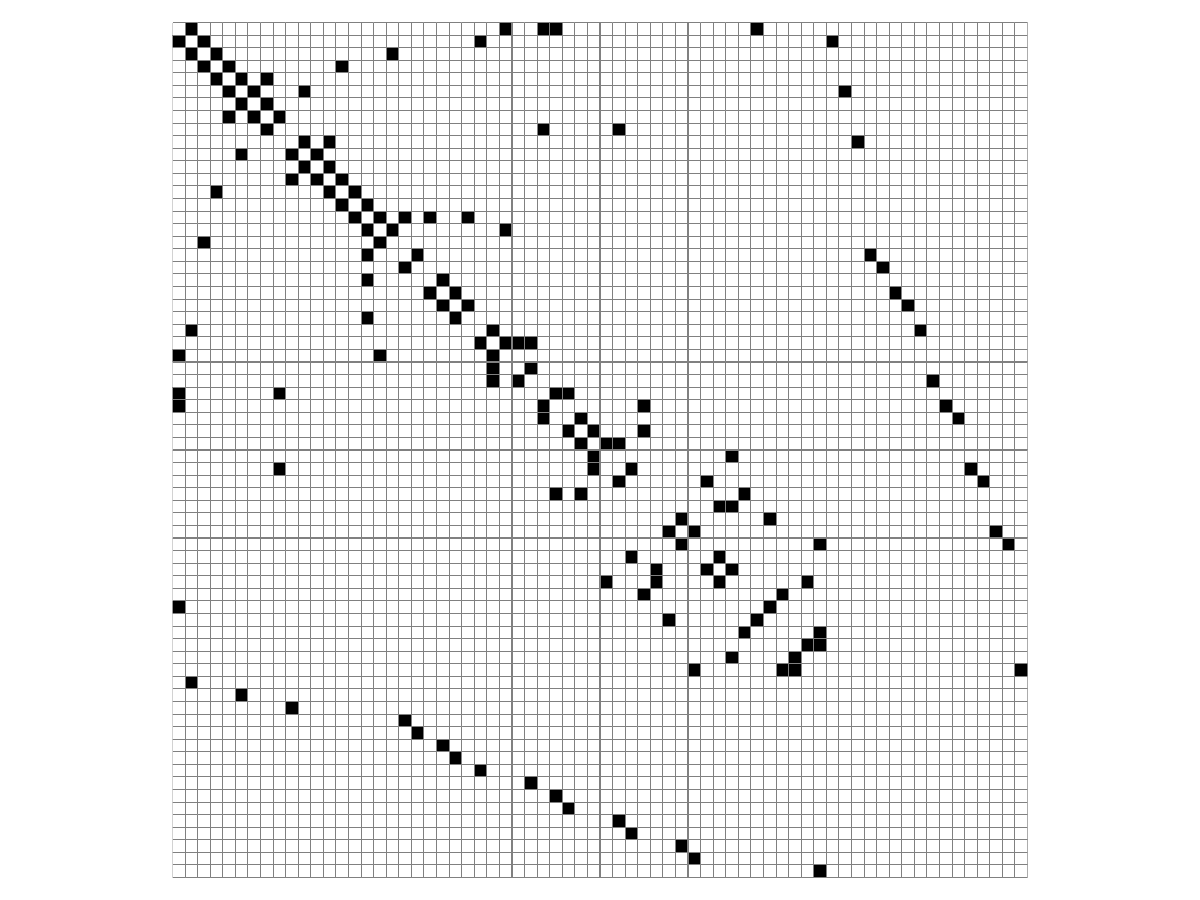}\label{fig:ieee-adj}}
    \caption{(a) Graph representation $\mathcal{G}$ of causality relationship in the IEEE 68-bus system. 
    (b) Adjacency matrix of the physical connection graph $\mathcal{G}$.}
    \label{fig:ieee-graph-adj}
\end{figure}

Given a multivariate time series of length \( t \), our goal is to develop a \textbf{model-intrinsic, explainable approach} that can: (1) Detect whether an anomaly exists; (2) Identify its type (measurement or cyber); (3) Localize the root-cause variable; (4) Characterize the anomaly’s shape over time.

To accomplish this, we propose a novel framework that learns and tracks \textit{causal relationships embedded in ODE dynamics} from time series data. Our method, described in the next section, leverages causal structure to achieve real-time, interpretable diagnostics for complex cyber-physical systems.

\begin{table*}[!h]
\small
\centering
\caption{Performance of PICODE compared to baseline models in anomaly detection.}
\begin{tabular}{llccccccccccccc}
    \midrule
    & & & \multicolumn{3}{c}{\textbf{Measurement Anomaly}} & \multicolumn{3}{c}{\textbf{Step Attack}} & \multicolumn{3}{c}{\textbf{Poisoning Attack}} \\
    \cmidrule(r){3-12}
    & & & Precision & Recall & F1 & Precision & Recall & F1 & Precision & Recall & F1  \\
    \midrule
    \multirow{3}{*}{} & Deep SVDD & & 0.656 & 0.667 & 0.661 & 0.898 & 0.912 & 0.905 & 0.102 & 0.103 & 0.103  \\
    & Anomaly Transformer & & 0.953 & 0.945 & 0.949 & 0.924 & 0.928 & 0.927 & 0.715 & 0.702 & 0.709 \\
    & PICODE & & 0.944 & 0.954 & 0.949 & 0.936 & 0.935 & 0.936 & 0.725 & 0.717 & 0.721 \\
    \midrule
\end{tabular}
\label{table:main-result}
\vspace{-0.1in}
\end{table*}

\section{Methodology}

Inspired by the concept of phase angle difference in power systems, we begin by encoding the voltage magnitude and angle into two channels:
\begin{equation}
    x_i = Vm_i \cdot \cos(Va_i), \quad y_i = Vm_i \cdot \sin(Va_i),
\end{equation}
where $Vm_i$ and $Va_i$ denote the voltage magnitude and angle at bus $i$. The full system state is then represented as $S(t) = [x_1(t), \dots, x_p(t), y_1(t), \dots, y_p(t)]^\top \in \mathbb{R}^{2p}$.

We propose the Interpretable Causality Ordinary Differential Equation (ICODE) network, a model-intrinsic, explainable framework for time series modeling and anomaly analysis in power systems. PICODE learns the underlying ODE dynamics of the system while also extracting the causal dependencies among variables. The model is defined as:
\begin{equation}
\label{eq:ICODE}
\begin{aligned} 
\frac{dS(t)}{dt} & = \Phi_{\theta}(S(t))S(t) + b, \\
\hat{S}(t+1) & = S(t) + \int_{t}^{t+1} \left( \Phi_{\theta}(S(\tau))S(\tau) + b\right) d\tau,
\end{aligned}
\end{equation}
where $b \in \mathbb{R}^{2p}$ is a bias term, and $\Phi_{\theta}(S(t)) \in \mathbb{R}^{2p \times 2p}$ is a neural network parameterized by $\theta$ that encodes causal dependencies among all transformed variables.

To interpret causality at the physical variable level, we aggregate relevant dependencies between $x_i$ and $y_i$. The final causality graph adjacency score $C_{i,j}$, representing the influence of bus $j$ on bus $i$, is computed as:
\begin{equation}
\begin{aligned}
C_{i,j} = \text{median}_{0 \leq t < T} \{|\Phi_{x_i \leftarrow x_j}| + |\Phi_{y_i \leftarrow x_j}| \\
+ |\Phi_{y_i \leftarrow y_j}| + |\Phi_{x_i \leftarrow y_j}| \},    
\end{aligned}
\end{equation}
where $\Phi_{a \leftarrow b}$ denotes the entry of $\Phi_\theta$ representing the influence from $b$ to $a$, and \textit{aggregate} refers to the transformation from $\Phi_{\theta}(S(t)) \in \mathbb{R}^{2p \times 2p}$ to $C \in \mathbb{R}^{p \times p}$ as described above. This aggregation enables reasoning over the original voltage variables despite encoding them into sine and cosine components.

For prediction, we employ Neural ODEs to estimate the integral and update the system state:
\begin{equation}
\hat{S}(t+1) = \text{NeuralODE}(S(t), (\Phi_{\theta}(S(t))S(t) + b), [t, t+1]).
\end{equation}

To ensure sparsity, we optimize the model with:
\begin{equation}
\begin{aligned}
Loss = \frac{1}{T-1} \sum_{t=0}^{T-1} \| \hat{S}(t+1) - S(t+1) \|^2 \\
+ \frac{\lambda}{T-1} \sum_{t=0}^{T-1} R(\Phi_{\theta}(S(t+1))).    
\end{aligned}
\label{eq:loss}
\end{equation}

\subsection{Anomaly Detection}
Anomalies are detected based on prediction deviation:
\begin{equation}
\label{eq:anomaly-score}
\text{AnomalyScore} = \sum_{t=0}^{T-1} \|\hat{S}(t) - S(t)\|_2.
\end{equation}

The AnomalyScore represents how the current pattern differs from the pattern during normal periods; thus, a higher score means a larger divergence between the current time period and the normal time period. Therefore, we employ a threshold to detect anomaly~\cite{xu2021anomaly}. 

\subsection{Root Cause Localization}
After we detect the anomaly and extract the anomaly time window, root cause localization is performed to localize the root cause of the anomaly. 
Intuitively, we localize the root cause by selecting the variable with the largest causality change. 
For both cyber and measurement anomaly, the root cause could be located by: 
\begin{equation}
\label{eq:get-root-cause}
\argmax_j \sum_{i=1}^p |C_{ij} - C'_{ij}|.
\end{equation}

While many models can perform the above two steps, identifying the root cause alone is often insufficient in cyber-physical systems like power grids. It is also crucial to determine whether the anomaly is local or likely to propagate. We next address this overlooked aspect.

\vspace{-0.1in}
\begin{table*}[!h]
\scriptsize
\centering
\caption{Performance of PICODE compared to baseline models in root cause localization.}
\begin{tabular}{llcccccccccc}
    \midrule
    & & & \multicolumn{3}{c}{\textbf{Measurement Anomaly}} & \multicolumn{3}{c}{\textbf{Poisoning Attack}} & \multicolumn{3}{c}{\textbf{Step Attack}} \\
    \cmidrule(r){4-12}
    & & & Top 1 & Top 3 & Top 5 & Top 1 & Top 3 & Top 5 & Top 1 & Top 3 & Top 5  \\
    \midrule
    \multirow{2}{*}{$V_m, V_a$} & CausalRCA &  & 0.015 & 0.044 & 0.073 & 0.015 & 0.088 & 0.088 & 0.044 & 0.088 & 0.088  \\    
    & RootClam  & & 0.015 & 0.044 & 0.073& 0.015 & 0.044 & 0.074 & 0.015 & 0.044 & 0.074 \\
    
    \midrule
    \multirow{2}{*}{$V_m \cos V_a, V_m \sin V_a$} & CausalRCA & &  0.015 & 0.044 & 0.073 & 0 & 0.015 & 0.029 & 0 & 0.044 & 0.088 \\
    & RootClam  & &  0.015 & 0.044 & 0.073 & 0.015 & 0.044 & 0.074 & 0.015 & 0.044 & 0.074 \\
    \midrule
    & PICODE & & 0.794 & 0.838 & 0.853 & 0.029 & 0.324 & 0.573 & 0.029 & 0.426 & 0.515\\
    \midrule
\end{tabular}
\label{table:rca-result}
\end{table*}

\subsection{Anomaly Type Classification}
We hypothesize that the type and structure of an anomaly can be inferred from the pattern of changes in causal relationships between normal and anomalous periods. In particular, a localized change—where the variation is concentrated in a single variable—typically indicates a \emph{measurement anomaly}, while distributed changes suggest a \emph{cyber anomaly}.

Let \( C \in \mathbb{R}^{p \times p} \) denote the causal matrix during the normal period and \( C' \) the matrix during the anomalous period. To measure the structural shift, we define row-wise change as:
\[
d_i = \frac{1}{p} \sum_{j=1}^p \left| C_{ij} - C'_{ij} \right|.
\]

Let \( i^* = \arg\max_i d_i \) denote the index of the row with the largest change. Then the top two changes are defined as:
\begin{align}
M_1(C, C') &= d_{i^*}, \\
M_2(C, C') &= \max_{i \neq i^*} d_i.
\end{align}

An anomaly is classified as a measurement anomaly if the dominance gap exceeds a threshold, representing that the changes are mostly in one column:
\begin{equation}
    \label{eq:anomaly-type}
    M_1(C, C') - M_2(C, C') \geq \gamma.
\end{equation}

Otherwise, it is classified as a cyber anomaly.

\subsection{Anomaly Shape Characterization}
To assess whether an anomaly is \emph{intensifying} or \emph{dissipating} over time, we partition the anomalous period into three consecutive time windows: \([0, t],\ [t, 2t],\ [2t, 3t]\). Within each window, we evaluate the deviation of the causal structure from normal state by tracking changes in the estimated causal graph.

Specifically, for a root cause variable \( i \), we define the cumulative causal shift within a window \([t_1, t_2]\) as:
\[
\Delta(t_1, t_2) = \sum_{j=1}^p \left| C_{ij} - C'_{ij} \right|,
\]
where \( C \) denotes the causal matrix under normal conditions and \( C' \) is the matrix learned from anomalous data within the interval \([t_1, t_2]\). The quantity \( \Delta(t_1, t_2) \) quantifies the total structural change affecting variable \( i \) during that time window.

By comparing \( \Delta(t_1, t_2) \) across sequential windows, we can identify the temporal evolution of the anomaly—whether its influence on the causal graph is growing, shrinking, or fluctuating. This enables a nuanced characterization of the anomaly’s \emph{shape} over time.

An overview of the PICODE architecture is provided in Figure~\ref{fig:architecture}. 
The theoretical analysis supporting the localization and type classification has already been discussed in the previous work. While the previous method was a general anomaly detector~\cite{sun2025unifying}, which was relatively limited in application for power systems. Our method effectively demonstrated its effectiveness in the power system exploration. In the following section, we delve into how PICODE captures and interprets anomaly shapes through temporal causal dynamics.

\section{Theoretical Analysis}

Analyzing black-box models can be difficult due to their lack of interpretability. To simplify the PICODE model, we assume that the Neural ODE effectively learns the dynamics inside the integral, which can be approximated by \( \Phi_\theta(S(t))S(t) \). We treat the causal graph adjacency score \( C \) as a \( p \times p \) matrix that solves the optimization problem defined by the loss in Eq.~(\ref{eq:loss}). This allows us to recast the learning process as a multi-task ridge regression problem, by assuming learning the $\Phi_{\theta}$ as an adjacency matrix $C$. 
For simplicity, we express the objective in the standard form of ridge regression as:
\begin{equation}
J(\bm{\beta}) = \|\mathbf{Y} - \mathbf{X}\bm{\beta}\|^2 + \lambda\|\bm{\beta}\|^2,
\end{equation}
where: $\mathbf{Y}$ is the matrix of target values, representing the adjacency of the causal graph, $\mathbf{X}$ is the matrix of features, $\bm{\beta}$ is the coefficient vector we're estimating, and $\lambda$ is the regularization parameter.
This is a general multi-variable ridge regression problem, and the analytical solution to this optimization problem is:

\begin{equation}
\label{eq:h_beta_ridge}
\hat{\bm{\beta}}_{ridge} = (\mathbf{X}^T\mathbf{X} + \lambda \mathbf{I})^{-1}\mathbf{X}^T\mathbf{Y},
\end{equation}
where $\hat{\bm{\beta}}_{ridge}$ is the solution to the ridge regression problem. Building on this result, we propose Theorem 1 to characterize the behavior of the solution under measurement perturbations.

\begin{theorem}[Monotonicity in Measurement Anomaly]
\label{theorem:measure}
Let $f(z) = \|\tilde{\beta}_j(z) - \beta_j\|^2$ represent the squared error in the $j$-th coefficient when the target vector is perturbed as $\tilde{\mathbf{Y}} = \mathbf{Y} + z\mathbf{v}$. If the original estimator produces a perfect estimate for $\beta_j$ (i.e., $\hat{\beta}_j = \beta_j$), then:
\begin{itemize}
    \item If $d_j \neq 0$, the function $f(z)$ is strictly monotonically increasing for $z > 0$ if $d_j > 0$, or for $z < 0$ if $d_j < 0$
    \item The function achieves its minimum value of 0 at $z = 0$
\end{itemize}
where $d_j$ is the $j$-th component of $\mathbf{d} = (\mathbf{X}^T\mathbf{X} + \lambda \mathbf{I})^{-1}\mathbf{X}^T\mathbf{v}$.
\end{theorem}

\begin{proof}
    See Appendix. 
\end{proof}

Theorem~\ref{theorem:measure} suggests that the output of the explainable module $\Phi_\theta$ can effectively reflect the magnitude of a measurement anomaly: larger measurement errors induce more significant changes in $\Phi_\theta$. This property is crucial for assessing anomaly urgency. If the anomaly-induced changes in $\Phi_\theta$ are not increasing, the anomaly may be benign and not require immediate intervention. For instance, due to the influence of sensor noise or instantaneous load fluctuations, the voltage or current values may vary within a short period of time. However, as the time steps progress, the data will tend to return to normal. In contrast, if the changes grow over time, even a localized measurement anomaly may escalate into a critical issue; for example, when the poisoning attack occurs, the error will accumulate, and the voltage values will keep exceeding the safe operational thresholds. 

In the following theorem, we extend this insight by showing that cyber anomaly also induce locally monotonic changes in the learned causal dynamics.

\begin{theorem}[Monotonicity of Cyber Anomaly]
\label{theorem:cyber}
When the original feature matrix provides a perfect estimate of $\beta_j$, the squared error $\|\tilde{\beta}_j(z) - \beta_j\|^2$ is:
\begin{itemize}
    \item Minimized at $z = 0$ with a value of $0$
    \item Strictly monotonically increasing for $z > 0$
    \item Strictly monotonically decreasing for $z < 0$
\end{itemize}
\end{theorem}

\begin{proof}
    See Appendix. 
\end{proof}

Theorem~\ref{theorem:cyber} demonstrates that cyber anomaly also induce locally monotonic changes at root cause variables in the learned causal structure captured by $\Phi_\theta$. Unlike measurement anomaly, which only affect observed values, cyber anomaly directly perturb the underlying system dynamics by altering the ODE itself. As a result, the change in $\Phi_\theta$ is more widespread and persistent. This monotonic progression provides a signal for assessing the evolving severity of the anomaly. A growing deviation in $\Phi_\theta$ over time suggests that functions injected during the cyber attack are changing and propagating through the system, potentially compromising multiple components. Identifying such trends is essential for timely mitigation, especially in cases like FDI attacks, where early intervention can prevent cascading failures.

\section{Experiments}
\label{experiments}
The distinct patterns of change in the causality relationships $|C - C'|$ for measurement and cyber anomaly provide a robust foundation for anomaly type classification. Measurement anomaly produce a concentrated change in the causality matrix, primarily affecting a single variable's relationships, aligning with scenarios like sensor errors or calibration issues. Conversely, cyber anomaly result in a more diffuse pattern of changes, reflecting their ability to propagate through interconnected systems. These theoretical insights not only support our anomaly classification approach but also provide a deeper understanding of how different types of anomaly manifest in dynamical systems, enabling more targeted and appropriate responses to system irregularities.
The purpose of our experiments is listed as follows: 
\begin{itemize}
    \item We evaluate PICODE's performance in identifying anomaly within power systems, comparing it to state-of-the-art methods. 
    \item We examine the model's capability to accurately pinpoint the location of the root cause, comparing it to state-of-the-art methods.
    \item  We demonstrate PICODE's unique ability to characterize the type and the shape of the anomaly, which is not provided in other works.
\end{itemize}

\begin{figure}
\centering
\includegraphics[width=0.99\columnwidth]{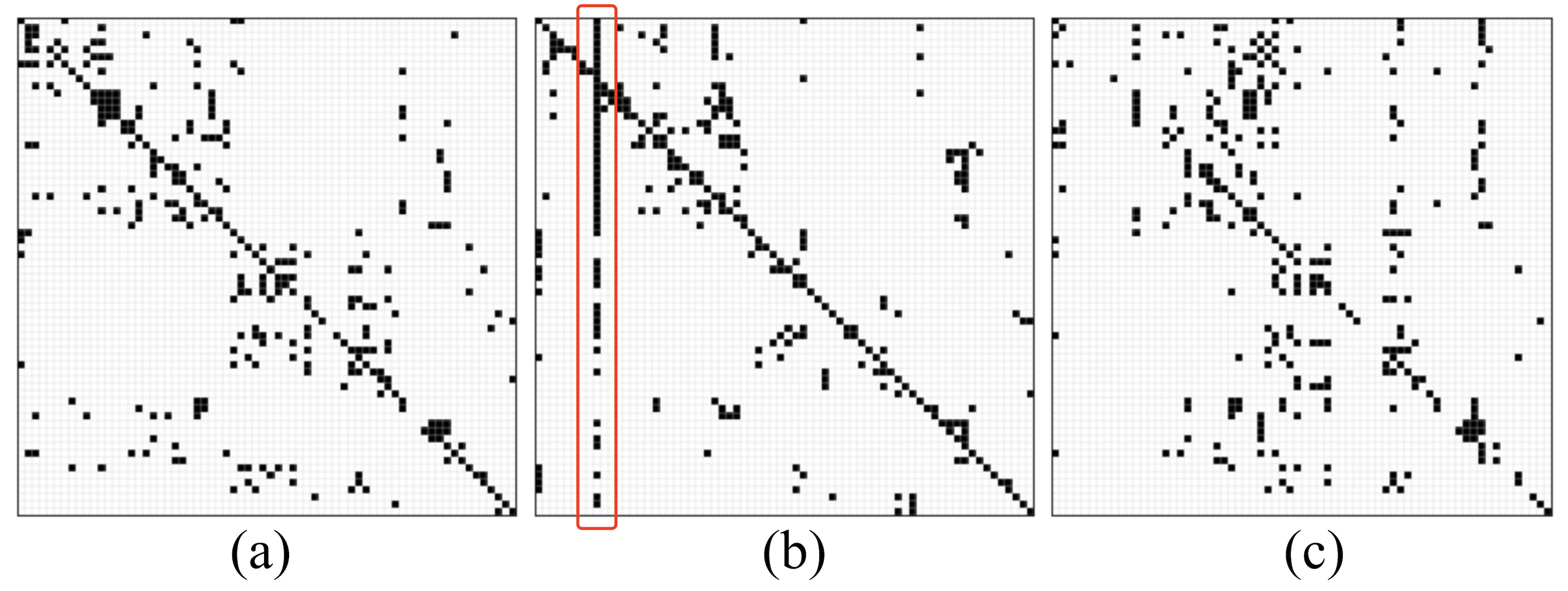} 
\caption{Binarized adjacency matrices of the learned causal graph for IEEE 68-bus system. (a) Normal periods. (b) Measurement anomaly with root cause at sensor 9. The impact is concentrated in a single column. (c) Cyber anomaly with root cause at sensor 9. Changes are distributed across all sensors.}
\label{fig:problem-def}
\end{figure}

\subsection{Experiment Settings}
We evaluate our approach using a simulated power system testbed developed by Pacific Northwest National Laboratory~\cite{nandanoori2020model}, which emulates cyber-physical attacks and measurement errors. The system comprises 68 components, each equipped with sensors that monitor various electrical variables, including voltage magnitude, voltage angle, frequency, current magnitude, and current angle. For our experiments, we focus on voltage magnitude and angle, two critical variables for power system stability, to detect anomaly and perform root cause analysis (RCA).
In our experiments, we enable the power system stabilizer and automatic generation control.

The spatial layout of the 68 components in the IEEE 68 system~\cite {zhao2014design}. is depicted in Figure~\ref{fig:ieee-diagram}, which is governed by a physical connectivity graph, illustrated with its corresponding adjacency matrix shown in Figure~\ref{fig:ieee-graph-adj}. We note that the causal graph in the IEEE 68 system could be different from the physical connection graph since each component is different. To reflect realistic operating conditions, we enable key control mechanisms, including load variations, power system stabilizers, and automatic generation control.

We simulate a time series with 3000 data points to learn the pattern during the regular time period. 
For anomaly detection, RCA tasks, type classification, and shape characterization, we simulate 68 time series for measurement anomaly traces, and 136 time series traces with cyber anomaly. Each dataset consists of 30,000 samples, each with 3,000 injected anomalies. These anomalies affect the root cause sensor and consequently influence the overall system. Measurement anomaly are simulated using Eq. (\ref{eq:physical}). Cyber anomaly, on the other hand, are generated using the GridSTAGE simulator and include two representative false data injection (FDI) patterns: step attacks and poisoning attacks, which will be discussed in the following section.

\subsubsection{Anomaly Simulation}
To simulate measurement anomaly, we perturb sensor readings using the following model:
\begin{equation}
\label{eq:anomaly-value}
    A(x_i) = x_i + z(t),
\end{equation}
where $z(t)$ is a monotonic or a unimodal function of time $t$, representing two measurement errors. This perturbation is applied post-simulation and affects only the observed data, without modifying system's underlying physical dynamics.

In contrast, cyber anomaly are injected directly during the time-series simulation using the GridSTAGE simulator. These include two widely studied types of false data injection (FDI) attacks: step attacks and poisoning attacks. Unlike measurement anomaly, cyber anomaly alter the system's trajectory by modifying internal variables during simulation, thereby disrupting the governing ODEs. This distinction reflects their respective impacts—measurement anomaly affect only observations, whereas cyber anomaly influence both the system's evolution and downstream measurements. We use a causal graph of the normal period as the base result to detect anomaly, locate the root cause, and characterize the shape.

\subsubsection{Baselines}
For anomaly detection, we compare our model with Deep SVDD~\cite{zhou2021vae} and Anomaly Transformer~\cite{xu2021anomaly} to demonstrate comparable performance on the simulated datasets. 
For root cause analysis, we benchmark against RootClam~\cite{han2023root} and CausalRCA~\cite{budhathoki2022causal}, which are state-of-the-art models in RCA.


\subsection{Results and Analysis}


We first evaluate the result using PICODE to learn the causal graph in the IEEE 68 system. We need to note that even though the result should be similar to the connection graph in the IEEE 68 system, since each component has different functions and is distinct from resistance and capacitance, the influence of each component is also different, which implies some connections between components might not exist because they are not significant enough. The adjacency matrix of the causal graph is shown in Figure~\ref{fig:problem-def}.

The adjacency matrix shows that the causal graph is a subgraph of the physical connection graph, with each component strongly related to itself, and some connections on the right are also meaningful. This means that the causal graph learned by PICODE is consistent with the physical nature of the IEEE 68 system, and PICODE successfully learns the spatial-temporal dependency. 

\subsubsection{Anomaly Detection}

We evaluate the anomaly detection performance of our model without distinguishing between different anomaly types. We train PICODE and baseline models using data from normal operating periods, then compute an anomaly score for each test sample using Eq.~\eqref{eq:anomaly-score}. We report detection precision, recall, and F1-score across three anomaly scenarios, with results summarized in Table~\ref{table:main-result}.

\small
\begin{table}[h!]
\centering
\caption{Type Classification and Shape Characterization}
\label{table:rca-class}
\begin{tabular}{lll}
    \toprule
      & Measurement Anomaly & Cyber Anomaly \\
    \midrule
    Classification & 0.824 & 1.000 \\
    Shape Characterization & 0.950 & 0.775 \\
    \bottomrule
\end{tabular}
\end{table}
\normalsize

As shown in Table~\ref{table:main-result}, our proposed model, \textbf{PICODE}, consistently outperforms baseline methods, including Deep SVDD and Anomaly Transformer, across all evaluation metrics. We attribute this performance gain to PICODE's ODE-based architecture, which effectively leverages domain-informed structural priors to learn the underlying system dynamics. In contrast, purely data-driven models such as Deep SVDD and Anomaly Transformer are less effective in this setting, 
This may be due to DeepSVDD's limited ability to utilize time series features and its relatively simple architecture, which struggles to capture the complex temporal correlations in dynamic power data.
The detection performance for the Poisoning Attack scenario is slightly lower than for other types of anomaly. This reduced performance is primarily attributed to the subtle nature of changes during the early stages of a poisoning attack. Since the ground-truth label marks the attack's onset while the observable deviation remains minimal initially, this temporal mismatch results in reduced accuracy across all models.

Overall, these results demonstrate that PICODE is capable of accurately modeling system behavior during normal operating periods and detecting deviations indicative of anomaly. This constitutes a critical first step toward explainable and reliable anomaly detection in cyber-physical systems.

\subsubsection{Root Cause Localization}

We evaluate PICODE's performance on root cause localization using synthetic datasets with injected anomaly. Building upon the causal model trained on normal operation data (as described in the previous section), we retrain PICODE using samples from each anomalous period to capture the updated causal structure. The root cause is then inferred by applying Eq.~(\ref{eq:get-root-cause}), which ranks variables based on their computed root cause scores.
Localization accuracy is evaluated using the Top-\(k\) metric, which checks whether the ground truth root cause appears in the top-\(k\) ranked variables. For fair comparison, both baseline methods (i.e., CausalRCA and RootClam) are augmented with the same domain-informed features used by PICODE. We do not report results from PICODE without these features, as its performance was significantly degraded without them.

As shown in Table~\ref{table:rca-result}, PICODE substantially outperforms both CausalRCA and RootClam across all three simulated datasets, achieving higher Top-1, Top-3, and Top-5 localization accuracy. These results support our hypothesis that explicitly modeling system dynamics through ordinary differential equations (ODEs) enhances both anomaly detection and root cause localization.
We believe PICODE’s better performance stems from its alignment with the physical principles governing dynamical systems. By directly learning the system's ODE parameters, PICODE can identify the variables most responsible for deviations, offering a more principled localization mechanism. 
CausalRCA struggles to localize the root cause in power systems because it fails to adapt to normal patterns during anomaly periods. In addition, RootCLAM employs a VAE-based model for root cause localization, which encodes the input data into a normal distribution, which may not hold in power systems, due to complex and domain-specific distributions. Moreover, the VAE framework typically requires a larger amount of training data to capture such complexity, especially in systems with additional architectures like load balancing.
In the next section, we focus on anomaly type classification, grounded in our theoretical framework.

\subsubsection{Anomaly Type Classification and Shape Characterization}

We evaluate the classification performance of PICODE in distinguishing between \textbf{Measurement Anomalies} and \textbf{Cyber Anomalies} (e.g., poisoning and step attacks), using the simulated dataset introduced earlier.

An anomaly is categorized as a measurement anomaly if the change metric $M(C, C') \geq 0.6$, otherwise, it is classified as a cyber anomaly, as shown in Figure~\ref{fig:problem-def}. The corresponding classification accuracies are shown in the first row of Table~\ref{table:rca-class}.

The results show that PICODE achieves high accuracy in identifying cyber anomaly, while measurement anomaly are sometimes misclassified as cyber anomaly, but cyber anomaly could hardly be classified as measurement anomaly. This is mainly due to the strict condition for identifying measurement anomaly, which requires observing a clear column-wise change in the dependency matrix. However, the classification function demonstrates strong performance in general.

To further characterize the shape of injected anomaly, we assume the root cause is known and introduce 20 unimodal and 20 monotonic functions for measurement anomaly, along with 20 trapezoid-shaped (unimodal) and 20 ramp (monotonic) attacks for cyber anomaly. Each anomaly spans 3000 time points and is divided into three windows to train PICODE independently. We use the metric $\Delta(t_1, t_2)$ to capture the anomaly’s shape signature. As shown in the second row of Table~\ref{table:rca-class}, PICODE accurately identifies the underlying shape, especially for measurement anomaly, demonstrating its ability to discern whether the system is recovering or deteriorating.

These findings support Theorems~\ref{theorem:measure} and \ref{theorem:cyber}, and collectively reinforce our hypothesis that changes in causal dependencies are informative not only for anomaly detection and root cause localization but also for anomaly classification and shape characterization. The interpretability offered by PICODE enhances situational awareness for domain experts and is particularly valuable for diagnostics in power systems.

\section{Conclusion}

In this work, we present PICODE, a unified and interpretable framework for anomaly detection, classification, root cause localization, and shape characterization in power systems. By modeling system dynamics with Neural ODEs and leveraging causality change insights, PICODE effectively captures both cyber and measurement anomaly behaviors. Experimental results validate its superiority over existing methods, and its integrated, explainable design makes it well-suited for management in cyber-physical systems.

\bibliographystyle{IEEEtran}
\bibliography{ICMLA2025}

\clearpage

\section{Appendix}
\setcounter{theorem}{0}
\begin{theorem}[Monotonicity in Measurement Anomaly]
\label{theorem:measure}
Let $f(z) = \|\tilde{\beta}_j(z) - \beta_j\|^2$ represent the squared error in the $j$-th coefficient when the target vector is perturbed as $\tilde{\mathbf{Y}} = \mathbf{Y} + z\mathbf{v}$. If the original estimator produces a perfect estimate for $\beta_j$ (i.e., $\hat{\beta}_j = \beta_j$), then:
\begin{itemize}
    \item If $d_j \neq 0$, the function $f(z)$ is strictly monotonically increasing for $z > 0$ if $d_j > 0$, or for $z < 0$ if $d_j < 0$
    \item The function achieves its minimum value of 0 at $z = 0$
\end{itemize}
where $d_j$ is the $j$-th component of $\mathbf{d} = (\mathbf{X}^T\mathbf{X} + \lambda \mathbf{I})^{-1}\mathbf{X}^T\mathbf{v}$.
\end{theorem}

\begin{proof}
    Starting with the expression for the perturbed coefficient:
\begin{equation}
\tilde{\beta}_j(z) = \hat{\beta}_j + zd_j
\end{equation}

Given that $\hat{\beta}_j = \beta_j$ (perfect original estimate), the squared error is:
\begin{align}
f(z) = \|\tilde{\beta}_j(z) - \beta_j\|^2 &= \|(\hat{\beta}_j + zd_j) - \beta_j\|^2 \\
&= \|zd_j\|^2 \\
&= z^2d_j^2
\end{align}

This is clearly a quadratic function of $z$. The derivative with respect to $z$ is:
\begin{equation}
\frac{df(z)}{dz} = 2zd_j^2
\end{equation}

From this derivative, we can observe that:
\begin{itemize}
    \item At $z = 0$, $\frac{df(z)}{dz} = 0$, indicating a stationary point
    \item For $z > 0$, $\frac{df(z)}{dz} > 0$ (assuming $d_j \neq 0$), indicating that $f(z)$ is strictly increasing
    \item For $z < 0$, $\frac{df(z)}{dz} < 0$ (assuming $d_j \neq 0$), indicating that $f(z)$ is strictly decreasing
\end{itemize}

The second derivative is:
\begin{equation}
\frac{d^2f(z)}{dz^2} = 2d_j^2 > 0 \text{ (for } d_j \neq 0 \text{)}
\end{equation}

This confirms that $f(z)$ is strictly convex and achieves its global minimum of 0 at $z = 0$. The function increases quadratically as $|z|$ increases in either direction.

Unlike the feature perturbation case, here the relationship between $z$ and the squared error is exactly quadratic, making the analysis more straightforward.
\end{proof}

\begin{theorem}[Monotonicity of Cyber Anomaly]
\label{theorem:cyber}
When the original feature matrix provides a perfect estimate of $\beta_j$, the squared error $\|\tilde{\beta}_j(z) - \beta_j\|^2$ is:
\begin{itemize}
    \item Minimized at $z = 0$ with a value of $0$
    \item Strictly monotonically increasing for $z > 0$
    \item Strictly monotonically decreasing for $z < 0$
\end{itemize}
\end{theorem}

\begin{proof}
Let $f(z) = \|\tilde{\beta}_j(z) - \beta_j\|^2$. Since we have a perfect estimate at $z=0$, we know $f(0) = 0$.

To establish monotonicity, we need to analyze the derivative $\frac{df(z)}{dz}$.

We start with the ridge estimator for the perturbed matrix:

\begin{equation}
\tilde{\boldsymbol{\beta}}(z) = (\tilde{\mathbf{X}}^T\tilde{\mathbf{X}} + \lambda \mathbf{I})^{-1}\tilde{\mathbf{X}}^T\mathbf{Y}
\end{equation}

Recall that $\tilde{\mathbf{X}} = \mathbf{X} + z\mathbf{e}_j$, where $\mathbf{e}_j$ is the unit vector in the $j$-th direction.

Let's denote $\mathbf{A}(z) = \tilde{\mathbf{X}}^T\tilde{\mathbf{X}} + \lambda \mathbf{I}$ and $\mathbf{b}(z) = \tilde{\mathbf{X}}^T\mathbf{Y}$, so $\tilde{\boldsymbol{\beta}}(z) = \mathbf{A}(z)^{-1}\mathbf{b}(z)$.

Taking the derivative with respect to $z$:
\begin{align}
\frac{d\tilde{\boldsymbol{\beta}}(z)}{dz} &= \frac{d}{dz}[\mathbf{A}(z)^{-1}\mathbf{b}(z)] \\
&= \frac{d\mathbf{A}(z)^{-1}}{dz}\mathbf{b}(z) + \mathbf{A}(z)^{-1}\frac{d\mathbf{b}(z)}{dz}
\end{align}

For the first term, we use the derivative of the inverse: $\frac{d\mathbf{A}^{-1}}{dz} = -\mathbf{A}^{-1}\frac{d\mathbf{A}}{dz}\mathbf{A}^{-1}$. This gives:
\begin{align}
\frac{d\tilde{\boldsymbol{\beta}}(z)}{dz} &= -\mathbf{A}(z)^{-1}\frac{d\mathbf{A}(z)}{dz}\mathbf{A}(z)^{-1}\mathbf{b}(z) + \mathbf{A}(z)^{-1}\frac{d\mathbf{b}(z)}{dz}
\end{align}

Now we compute the derivatives of $\mathbf{A}(z)$ and $\mathbf{b}(z)$:

\begin{align}
\frac{d\mathbf{A}(z)}{dz} &= \frac{d}{dz}[\tilde{\mathbf{X}}^T\tilde{\mathbf{X}} + \lambda \mathbf{I}] \\
&= \frac{d}{dz}[(\mathbf{X} + z\mathbf{e}_j)^T(\mathbf{X} + z\mathbf{e}_j) + \lambda \mathbf{I}] \\
&= \mathbf{e}_j^T(\mathbf{X} + z\mathbf{e}_j) + (\mathbf{X} + z\mathbf{e}_j)^T\mathbf{e}_j \\
&= \mathbf{e}_j^T\mathbf{X} + z\mathbf{e}_j^T\mathbf{e}_j + \mathbf{X}^T\mathbf{e}_j + z\mathbf{e}_j^T\mathbf{e}_j \\
&= \mathbf{e}_j^T\mathbf{X} + \mathbf{X}^T\mathbf{e}_j + 2z
\end{align}

And for $\mathbf{b}(z)$:
\begin{align}
\frac{d\mathbf{b}(z)}{dz} &= \frac{d}{dz}[\tilde{\mathbf{X}}^T\mathbf{Y}] \\
&= \frac{d}{dz}[(\mathbf{X} + z\mathbf{e}_j)^T\mathbf{Y}] \\
&= \mathbf{e}_j^T\mathbf{Y}
\end{align}

For the specific case where we're interested in $\tilde{\beta}_j(z)$, let's define $\mathbf{c}_j^T$ as the $j$-th row of $\mathbf{A}(z)^{-1}$. Then:
\begin{align}
\frac{d\tilde{\beta}_j(z)}{dz} &= -\mathbf{c}_j^T\frac{d\mathbf{A}(z)}{dz}\mathbf{A}(z)^{-1}\mathbf{b}(z) + \mathbf{c}_j^T\frac{d\mathbf{b}(z)}{dz} \\
&= -\mathbf{c}_j^T(\mathbf{e}_j^T\mathbf{X} + \mathbf{X}^T\mathbf{e}_j + 2z)\tilde{\boldsymbol{\beta}}(z) + \mathbf{c}_j^T\mathbf{e}_j^T\mathbf{Y}
\end{align}

Now we need to evaluate this at $z=0$. Note that at $z=0$, $\tilde{\mathbf{X}} = \mathbf{X}$ and $\tilde{\boldsymbol{\beta}}(0) = \hat{\boldsymbol{\beta}} = \boldsymbol{\beta}$ (since we assume the original estimate is perfect).

Let $\mathbf{c}_j^T(0)$ be the $j$-th row of $(\mathbf{X}^T\mathbf{X} + \lambda \mathbf{I})^{-1}$ at $z=0$.

Then:
\begin{align}
\left.\frac{d\tilde{\beta}_j(z)}{dz}\right|_{z=0} &= -\mathbf{c}_j^T(0)(\mathbf{e}_j^T\mathbf{X} + \mathbf{X}^T\mathbf{e}_j)\boldsymbol{\beta} + \mathbf{c}_j^T(0)\mathbf{e}_j^T\mathbf{Y} \\
&= -\mathbf{c}_j^T(0)(\mathbf{e}_j^T\mathbf{X} + \mathbf{X}^T\mathbf{e}_j)\boldsymbol{\beta} + \mathbf{c}_j^T(0)\mathbf{e}_j^T\mathbf{X}\boldsymbol{\beta} \\
&= -\mathbf{c}_j^T(0)(\mathbf{e}_j^T\mathbf{X} + \mathbf{X}^T\mathbf{e}_j - \mathbf{e}_j^T\mathbf{X})\boldsymbol{\beta} \\
&= -\mathbf{c}_j^T(0)\mathbf{X}^T\mathbf{e}_j\boldsymbol{\beta}
\end{align}

The sign of this derivative at $z=0$ depends on the specific data configuration. However, if we look at $f(z) = \|\tilde{\beta}_j(z) - \beta_j\|^2$, we can show that:

\begin{equation}
\frac{df(z)}{dz} = 2(\tilde{\beta}_j(z) - \beta_j)\frac{d\tilde{\beta}_j(z)}{dz}
\end{equation}

At $z=0$, we have $\tilde{\beta}_j(0) - \beta_j = 0$, making $\left.\frac{df(z)}{dz}\right|_{z=0} = 0$.

For the second derivative:
\begin{align}
\frac{d^2f(z)}{dz^2} &= 2\left(\frac{d\tilde{\beta}_j(z)}{dz}\right)^2 + 2(\tilde{\beta}_j(z) - \beta_j)\frac{d^2\tilde{\beta}_j(z)}{dz^2}
\end{align}

At $z=0$, this becomes:
\begin{align}
\left.\frac{d^2f(z)}{dz^2}\right|_{z=0} &= 2\left(\left.\frac{d\tilde{\beta}_j(z)}{dz}\right|_{z=0}\right)^2
\end{align}

This is strictly positive unless $\left.\frac{d\tilde{\beta}_j(z)}{dz}\right|_{z=0} = 0$, which would only happen in special degenerate cases.

Therefore, $f(z)$ has a strict local minimum at $z=0$ and is strictly convex in a neighborhood of $z=0$. Given that $f(0) = 0$ and $f(z) \geq 0$ for all $z$ (since it's a squared quantity), we conclude that:
\begin{itemize}
    \item $f(z)$ is strictly increasing for $z > 0$
    \item $f(z)$ is strictly increasing for $z < 0$ (as $|z|$ increases)
\end{itemize}

This establishes the monotonicity property with respect to $|z|$.
\end{proof}

\end{document}